\documentclass[times,twocolumn,final,authoryear]{elsarticle}

\usepackage{amsmath} 
\usepackage{prletters}
\usepackage{framed,multirow}

\usepackage{amssymb}
\usepackage{latexsym}

\usepackage{url}
\usepackage{xcolor}
\definecolor{newcolor}{rgb}{.8,.349,.1}

\journal{Pattern Recognition Letters}

\usepackage[round]{natbib}
\usepackage{bbm}
\usepackage{breqn}
\usepackage{hyperref}
\usepackage{amsthm}
\usepackage{placeins}

\usepackage{booktabs} 
\usepackage{multirow} 

\usepackage{atbegshi,picture}

\newtheorem{theorem}{Theorem}

\usepackage{stackengine}

\AtBeginShipout{\AtBeginShipoutFirst{%
  \put(\dimexpr\paperwidth-4.7cm\relax,-1.1cm){\makebox[0pt][r]{{Published in \textit{Pattern Recognition Letters}, (2020), 325-331, 140; \url{https://doi.org/10.1016/j.patrec.2020.11.008}}}}%
}}

\begin{document}

\ifpreprint
  \setcounter{page}{1}
\else
  \setcounter{page}{1}
\fi

\begin{frontmatter}

\title{Rank consistent ordinal regression for neural networks with application to age estimation}

\author[1]{Wenzhi \snm{Cao}} 
\author[2]{Vahid \snm{Mirjalili}}
\author[1]{Sebastian \snm{Raschka}\corref{cor1}}
\cortext[cor1]{Corresponding author:}
\ead{sraschka@wisc.edu}

\address[1]{University of Wisconsin-Madison, Department of Statistics, 1300 University Ave, Madison, WI 53705, USA}
\address[2]{Michigan State University, Department of Computer Science \& Engineering, 428 South Shaw Lane, East Lansing, MI 48824, USA}


\begin{abstract}
In many real-world prediction tasks, class labels include information about the relative ordering between labels, which is not captured by commonly-used loss functions such as multi-category cross-entropy. Recently, the deep learning community adopted ordinal regression frameworks to take such ordering information into account. Neural networks were equipped with ordinal regression capabilities by transforming ordinal targets into binary classification subtasks. However, this method suffers from inconsistencies among the different binary classifiers. 
To resolve these inconsistencies, we propose the COnsistent RAnk Logits (CORAL) framework with strong theoretical guarantees for rank-monotonicity and consistent confidence scores. Moreover, the proposed method is architecture-agnostic and can extend arbitrary state-of-the-art deep neural network classifiers for ordinal regression tasks. The empirical evaluation of the proposed rank-consistent method on a range of face-image datasets for age prediction shows a substantial reduction of the prediction error compared to the reference ordinal regression network.
\end{abstract}


\end{frontmatter}


\section{Introduction}
\label{sec:introduction}

Ordinal regression (also called ordinal classification), describes the task of predicting labels on an ordinal scale. Here, a ranking rule or classifier $h$ maps each object $\mathbf{x}_i \in \mathcal{X}$ into an ordered set ${h: \mathcal{X} \rightarrow \mathcal{Y}}$, where ${\mathcal{Y}=\{r_1  \prec ... \prec r_K\}}$. In contrast to classification, the labels provide enough information to order objects. However, as opposed to metric regression, the difference between label values is arbitrary.

While the field of machine learning has developed many powerful algorithms for predictive modeling, most algorithms have been designed for classification tasks. The extended binary classification approach proposed by~\cite{li2007ordinal} forms the basis of many ordinal regression implementations. However, neural network-based implementations of this approach commonly suffer from classifier inconsistencies among the binary rankings~\citep{niu2016ordinal}.
This inconsistency problem among the predictions of individual binary classifiers is illustrated in Figure~\ref{fig:inconsistency-issue}.
We propose a new method and theorem for guaranteed classifier consistency that can easily be implemented in various neural network architectures.

Furthermore, along with the theoretical rank-consistency guarantees, this paper presents an empirical analysis of our approach to challenging real-world datasets for predicting the age of individuals from face images using our method with convolutional neural networks (CNNs). Aging can be regarded as a non-stationary process since age progression effects appear differently depending on the person's age. During childhood, facial aging is primarily associated with changes in the shape of the face, whereas aging during adulthood is defined mainly by changes in skin texture~\citep{ramanathan2009age,niu2016ordinal}. Based on this assumption, age prediction can be modeled using ordinal regression-based approaches~\citep{yang2010ranking,chang2011ordinal,cao2012human,li2012learning}.

The main contributions of this paper are as follows:

\begin{enumerate}
    \item The consistent rank logits (CORAL) framework for ordinal regression with theoretical guarantees for classifier consistency;
    \item Implementation of CORAL to adapt common CNN architectures, such as ResNet~\citep{he2016deep}, for ordinal regression;
    \item Experiments on different age estimation datasets showing that CORAL's guaranteed binary classifier consistency improves predictive performance compared to the reference framework for ordinal regression.
\end{enumerate}

Note that this work focuses on age estimation to study the proposed method's efficacy for ordinal regression. However, the proposed technique can be used for other ordinal regression problems, such as crowd-counting, depth estimation, biological cell counting, customer satisfaction, and others.

\begin{figure}
\begin{center}
\centerline{\includegraphics[width=0.98\linewidth]{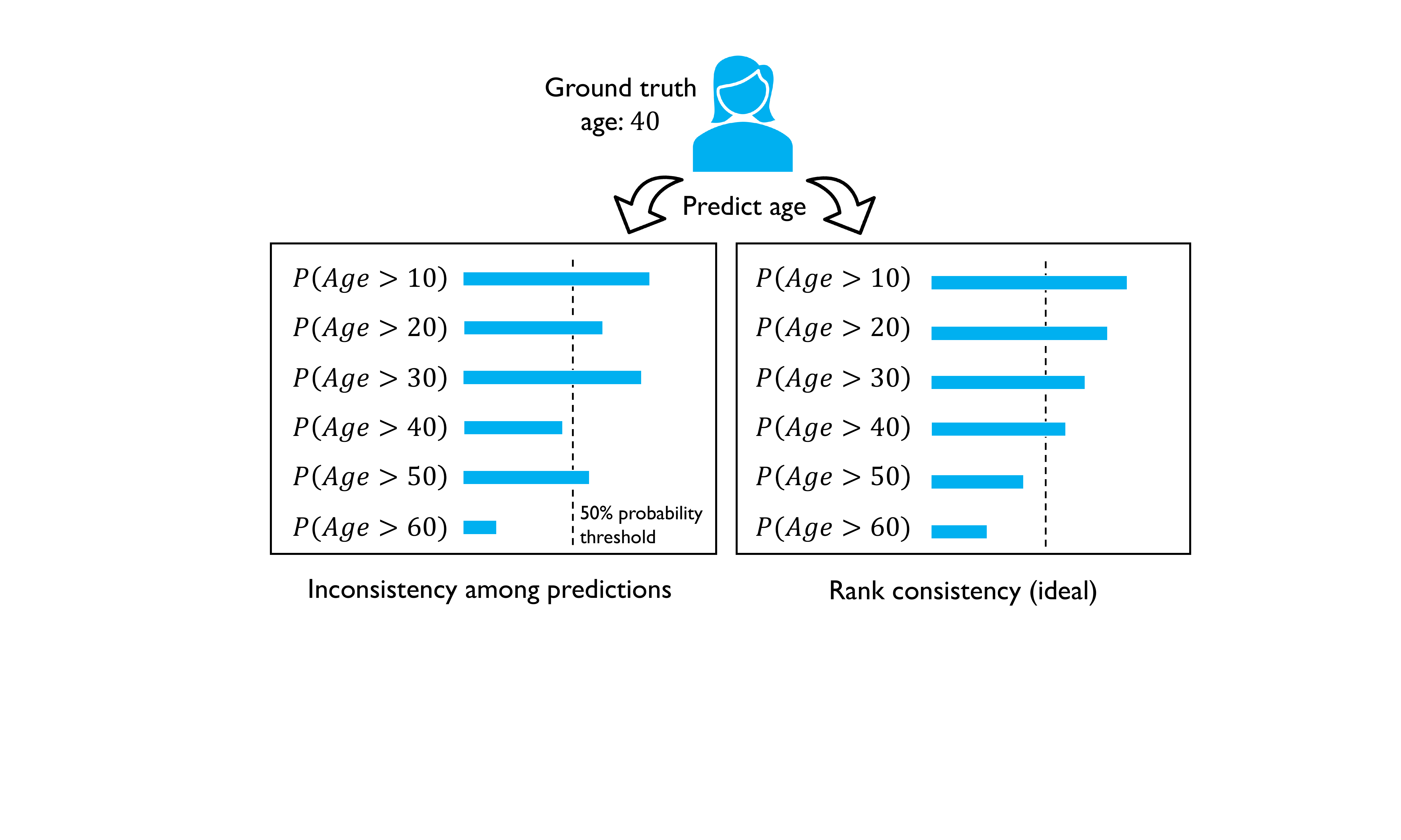}}
\caption{Schematic illustration of inconsistencies that can occur among individual classifiers in the general reduction framework from ordinal regression to binary classification: a rank-inconsistent model (left) versus a rank-consistent model where the probabilities decrease consistently (right)}
\label{fig:inconsistency-issue}
\end{center}
\end{figure}

\section{Related work}
\label{sec:related-work}

\subsection{Ordinal regression and ranking}

Several multivariate extensions of generalized linear models have been developed for ordinal regression in the past, including the popular proportional odds and proportional hazards models~\citep{mccullagh1980regression}. Moreover, the machine learning field developed ordinal regression models based on extensions of well-studied classification algorithms, by reformulating the problem to utilize multiple binary classification tasks~\citep{baccianella2009evaluation}. Early work in this regard includes the use of perceptrons~\citep{crammer2002pranking,shen2005ranking} and support vector machines~\citep{herbrich1999support,shashua2003ranking,rajaram2003classification,chu2005new}. \cite{li2007ordinal} proposed a general reduction framework that unified the view of a number of these existing algorithms.

\subsection{Ordinal regression CNN}

While earlier works on using CNNs for ordinal targets have employed conventional classification approaches~\citep{levi2015age,rothe2015dex}, the general reduction framework from ordinal regression to binary classification by \cite{li2007ordinal} was recently adopted by \cite{niu2016ordinal} as \textit{Ordinal Regression CNN} (OR-CNN). In the OR-CNN approach, an ordinal regression problem with $K$ ranks is transformed into $K-1$ binary classification problems, with the $k$-th task predicting whether the age label of a face image exceeds rank $r_k$, ${k=1,...,K-1}$. All $K-1$ tasks share the same intermediate layers but are assigned distinct weight parameters in the output layer. 

While the OR-CNN was able to achieve state-of-the-art performance on benchmark datasets, it does not guarantee consistent predictions, such that predictions for individual binary tasks may disagree. For example, in an age estimation setting, it would be contradictory if the $k$-th binary task predicted that the age of a person was more than 30, but a previous task predicted the person's age was less than 20. This inconsistency could be suboptimal when the $K-1$ task predictions are combined to obtain the estimated age.

\cite{niu2016ordinal} acknowledged the classifier inconsistency as not being ideal and also noted that ensuring the $K-1$ binary classifiers are consistent would increase the training complexity substantially~\citep{niu2016ordinal}. The CORAL method proposed in this paper addresses both these issues with a theoretical guarantee for classifier consistency and without increasing the training complexity.

\begin{figure*}
\begin{center}
\centerline{\includegraphics[width=0.98\linewidth]{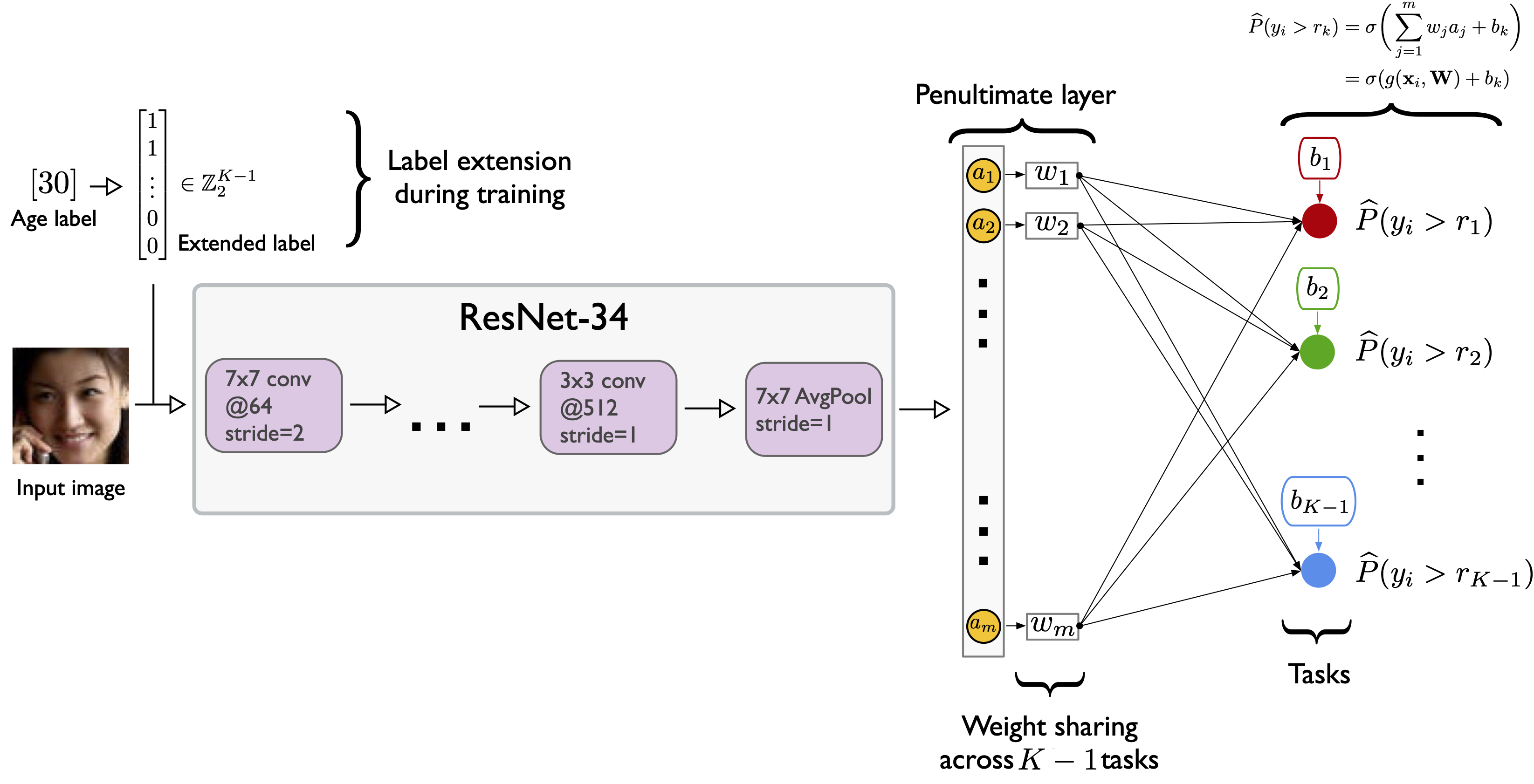}}
\caption{Illustration of the consistent rank logits CNN (CORAL-CNN) used for age prediction. From the estimated probability values, the binary labels are obtained via Eq.~\ref{eq:proba-to-binary} and converted to the age label via Eq.~\ref{eq:predicted-label}.}
\label{fig:resnet}
\end{center}
\end{figure*}

\subsection{Other CNN architectures for age estimation}

\cite{chen2017using} proposed a modification of the OR-CNN~\citep{niu2016ordinal}, known as Ranking-CNN, that uses an ensemble of CNNs for binary classifications and aggregates the predictions to estimate the age label of a given face image.  The researchers showed that training an ensemble of CNNs improves the predictive performance over a single CNN with multiple binary outputs~\citep{chen2017using}, which is consistent with the well-known fact that an ensemble model can achieve better generalization performance than each individual classifier in the ensemble~\citep{raschka2019pyml3}.

Recent research has also shown that training a multi-task CNN that shares lower-layer parameters for various face analysis tasks (face detection, gender prediction, age estimation, etc.) can improve the overall performance across different tasks compared to a single-task CNN~\citep{ranjan2017all}. 

Another approach for utilizing binary classifiers for ordinal regression is the siamese CNN architecture proposed by~\cite{polania2018ordinal}, which computes the rank from pair-wise comparisons between the input image and multiple, carefully selected anchor images.

\section{Proposed method}
\label{sec:proposed}

This section describes our proposed CORAL framework that addresses the problem of classifier inconsistency in the OR-CNN by~\cite{niu2016ordinal}, which is based on multiple binary classification tasks for ranking.

\subsection{Preliminaries}

Let ${D=\{\mathbf{x}_i,y_i\}_{i=1}^N}$ be the training dataset consisting of $N$ training examples. Here, $\mathbf{x}_i\in \mathcal{X}$ denotes the $i$-th training example and $y_i$ the corresponding rank, where ${y_i\in \mathcal{Y}=\{r_1,r_2,...r_K\}}$ with ordered rank ${r_K\succ r_{K-1}\succ \ldots\succ r_1}$.  The ordinal regression task is to find a ranking rule $h: \mathcal{X}\rightarrow \mathcal{Y}$ such that a loss function $L(h)$ is minimized. 

Let $\mathcal{C}$ be a $K\times K$ \emph{cost matrix}, where $\mathcal{C}_{y,r_k}$ is the cost of predicting an example $(\mathbf{x},y)$ as rank $r_k$~\citep{li2007ordinal}. Typically, $\mathcal{C}_{y,y}=0$ and $\mathcal{C}_{y,r_k}>0$ for $y\neq r_k$. In ordinal regression, we generally prefer each row of the cost matrix to be \emph{V-shaped}, that is, ${\mathcal{C}_{y,r_{k-1}}\geq \mathcal{C}_{y,r_k}}$ if $r_{k}\leq y$  and ${\mathcal{C}_{y,r_k}\leq \mathcal{C}_{y,r_{k+1}}}$ if $r_{k}\geq y$. The \emph{classification} cost matrix has entries ${\mathcal{C}_{y,r_k}=\mathbbm{1}\{y\neq r_k\}}$ that do not consider ordering information. In ordinal regression, where the ranks are treated as numerical values, the \emph{absolute} cost matrix is commonly defined by ${\mathcal{C}_{y,r_k}=|y-r_k|}$.

\cite{li2007ordinal} proposed a general reduction framework for extending an ordinal regression problem into several binary classification problems. This framework requires a cost matrix that is convex in each row (${\mathcal{C}_{y,r_{k+1}}-\mathcal{C}_{y,r_k}\geq \mathcal{C}_{y,r_{k}}-\mathcal{C}_{y,r_{k-1}}}$ for each $y$) to obtain a rank-monotonic threshold model. Since the cost-related weighting of each binary task is specific for each training example, this approach is considered as infeasible in practice due to its high training complexity~\citep{niu2016ordinal}. 

Our proposed CORAL framework does neither require a cost matrix with convex-row conditions nor explicit weighting terms that depend on each training example to obtain a rank-monotonic threshold model and produce consistent predictions for each binary task.

\subsection{Ordinal regression with a consistent rank logits model}
 
In this section, we describe our proposed consistent rank logits (CORAL) framework for ordinal regression. Subsection~\ref{sec:label-extension} describes the label extension into binary tasks used for rank prediction. The loss function of the CORAL framework is described in Subsection~\ref{sec:loss}. In subsection~\ref{sec:theoretical-guarantees}, we prove the theorem for rank consistency among the binary classification tasks that guarantee that the binary tasks produce consistently ranked predictions.

\subsubsection{Label extension and rank prediction} 
\label{sec:label-extension}

Given a training dataset $D=\{\mathbf{x}_i,y_i\}_{i=1}^N$, a rank  $y_i$ is first extended into $K-1$ binary labels $y_i^{(1)},\ldots,y_i^{(K-1)}$ such that ${y_i^{(k)} \in \{0,1\}}$ indicates whether $y_i$ exceeds rank $r_k$, for instance, ${y_i^{(k)}=\mathbbm{1}\{y_i>r_k\}}$. The indicator function $\mathbbm{1}\{\cdot\}$ is $1$ if the inner condition is true and $0$ otherwise. Using the extended binary labels during model training, we train a single CNN with $K-1$ binary classifiers in the output layer, which is illustrated in Figure~\ref{fig:resnet}. 

Based on the binary task responses, the predicted rank label for an input $\mathbf{x}_i$ is obtained via $h(\mathbf{x}_i)=r_q$. The rank index\footnote{While the rank label $r_q$ is application-specific and defined by the user, for example {$r_q\in\{"bad", "okay", "good"\}$} or {$r_q\in\{18\,years, 19\,years, ... 70\,years\}$}, the rank index $q$ is an integer in the range $\{1, 2, ..., K\}$.} $q$ is given by
\begin{equation} \label{eq:predicted-label}
q=  1 + \sum_{k=1}^{K-1} f_k(\mathbf{x}_i),
\end{equation}
where $f_k(\mathbf{x}_i)\in \{0,1\}$ is the prediction of the $k$-th binary classifier in the output layer. We require that $\{f_k\}_{k=1}^{K-1}$ reflect the ordinal information and are  \emph{rank-monotonic}, ${f_1(\mathbf{x}_i)\geq f_2(\mathbf{x}_i) \geq \ldots \geq f_{K-1}(\mathbf{x}_i)}$, which guarantees consistent predictions. To achieve rank-monotonicity and guarantee binary classifier consistency (Theorem~\ref{th:ordered_thres}), the $K-1$ binary tasks share the same weight parameters\footnote{To provide further intuition for the weight sharing requirement, we may consider a simplified version, that is, the linear form $logit(p_i) = w x + b_i$ or $p_i = \sigma(w x + b_i)$ with a single feature $x$. If the weight $w$ is not shared across the $K-1$ equations, the S-shaped curves of the probability scores $p_i$ will intersect, making the p_i`s non-monotone at some given input $x$. Only if $w$ is shared across the $K-1$ equations, the S-shaped curves are horizontally shifted without intersecting.} but have independent bias units (Figure~\ref{fig:resnet}).

\subsubsection{Loss function} 
\label{sec:loss}

Let $\mathbf{W}$ denote the weight parameters of the neural network excluding the bias units of the final layer. The penultimate layer, whose output is denoted as $g(\mathbf{x}_i, \mathbf{W})$, shares a single weight with all nodes in the final output layer; $K-1$ independent bias units are then added to $g(\mathbf{x}_i,\mathbf{W})$ such that ${\{g(\mathbf{x}_i,\mathbf{W})+b_k\}_{k=1}^{K-1}}$ are the inputs to the corresponding binary classifiers in the final layer. Let 
\begin{equation}
\sigma(z)=1/(1+\exp(-z))
\end{equation}
be the logistic sigmoid function. The predicted empirical probability for task $k$ is defined as
\begin{dmath}
\widehat{P}(y_i^{(k)}=1)=\sigma(g(\mathbf{x}_i,\mathbf{W})+b_k).
\end{dmath}

For model training, we minimize the loss function
\begin{dmath}\label{eq:loss_fun}
    L(\mathbf{W},\mathbf{b})=\\
    - \sum_{i=1}^N\sum_{k=1}^{K-1} \lambda^{(k)}
    \big[ \log(\sigma(g(\mathbf{x}_i,\mathbf{W})+b_k))y_i^{(k)}  \\
    + \log(1-\sigma(g(\mathbf{x}_i,\mathbf{W})+b_k))(1-y_i^{(k)}) \big],
\end{dmath}

which is the weighted cross-entropy of $K-1$ binary classifiers. For rank prediction (Eq.~\ref{eq:predicted-label}), the binary labels are obtained via
\begin{equation} \label{eq:proba-to-binary}
    f_k(\mathbf{x}_i) =\mathbbm{1}\{\widehat{P}(y_i^{(k)}=1)>0.5\}.
\end{equation}

In Eq.~\ref{eq:loss_fun}, $\lambda^{(k)}$ denotes the weight of the loss associated with the $k$-th classifier (assuming $\lambda^{(k)}>0$). In the remainder of the paper, we refer to $\lambda^{(k)}$ as the importance parameter for task $k$. Some tasks may be less robust or harder to optimize, which can be considered by choosing a non-uniform task weighting scheme. For simplicity, we carried out all experiments with uniform task weighting, that is, $\forall k: \lambda^{(k)} = 1$. In the next section, we provide the theoretical guarantee for classifier consistency under uniform and non-uniform task importance weighting given that the task importance weights are positive numbers.

\subsubsection{Theoretical guarantees for classifier consistency} 
\label{sec:theoretical-guarantees}

The following theorem shows that by minimizing the loss $L$ (Eq.~\ref{eq:loss_fun}), the learned bias units of the output layer are non-increasing such that 

\begin{equation}
{b_1\geq b_2\geq \ldots \geq b_{K-1}}.
\end{equation}

Consequently, the predicted confidence scores or probability estimates of the $K-1$ tasks are decreasing, for instance, 

\begin{equation}
	\widehat{P}\left(y_i^{(1)}=1\right)\geq \widehat{P}\left(y_i^{(2)}=1\right) \geq \ldots \geq \widehat{P}\left(y_i^{(K-1)}=1\right)
\end{equation}

\noindent for all $i$, ensuring classifier consistency. Consequently, $\{f_k\}_{k=1}^{K-1}$ (Eq.~\ref{eq:proba-to-binary}) are also rank-monotonic.
\begin{theorem}[Ordered bias units]\label{th:ordered_thres}
 By minimizing the loss function defined in Eq.~\ref{eq:loss_fun}, the optimal solution $(\mathbf{W}^*,\mathbf{b}^*)$ satisfies $b_1^*\geq b_2^* \geq \ldots \geq b_{K-1}^*$. 
\end{theorem}

\begin{proof}
Suppose $(\mathbf{W},b)$ is an optimal solution and {$b_k < b_{k+1}$} for some $k$. Claim: replacing $b_k$ with $b_{k+1}$ , or replacing $b_{k+1}$ with $b_k$, decreases the objective value $L$. Let 
\begin{align*}
& A_1 =\{n: y_n^{(k)}=y_n^{(k+1)}=1\}, 
\\ 
& A_2=\{n: y_n^{(k)} = y_n^{(k+1)}=0\}, \\
& A_3 =\{n: y_n^{(k)}=1, \, y_n^{(k+1)}=0\}.
\end{align*}
By the ordering relationship, we have 
\begin{equation*}
A_1\cup A_2 \cup A_3= \{1,2,\ldots,N\}. 
\end{equation*}
\noindent Denote $p_n(b_k)= \sigma(g(\mathbf{x}_n,\mathbf{W})+b_k)$ and 
\begin{align*}
& \delta_{n} =  \log(p_n(b_{k+1}))-\log(p_n(b_k)), \\
& \delta\,'_{n} = \log(1-p_n(b_{k}))-\log(1-p_n(b_{k+1})).
\end{align*}
Since $p_n(b_k)$ is increasing in $b_k$, we have
$\delta_{n}>0$ and $\delta\,'_n>0$. 

\noindent If we replace $b_k$ with $b_{k+1}$, the loss terms related to the $k$-th task are updated.
The change of loss $L$ (Eq.~\ref{eq:loss_fun}) is given as
\begin{equation*}
\Delta_1 L = \lambda^{(k)}\Big[- \sum_{n\in A_1} \delta_{n} + \sum_{n\in A_2} \delta\,'_n - \sum_{n\in A_3} \delta_{n}\Big].
\end{equation*}

\noindent Accordingly, if we replace $b_{k+1}$ with $b_k$, the change of $L$ is given as 
\begin{equation*}
 \Delta_2 L = \lambda^{(k+1)}\Big[\sum_{n\in A_1}\delta_{n} - \sum_{n\in A_2} \delta\,'_n - \sum_{n\in A_3} \delta\,'_n \Big].
\end{equation*}

\noindent By adding $\frac{1}{\lambda^{(k)}}\Delta_1 L$ and $\frac{1}{\lambda^{(k+1)}}\Delta_2 L$, we have
\begin{equation*}
    \frac{1}{\lambda^{(k)}}\Delta_1 L+ \frac{1}{\lambda^{(k+1)}}\Delta_2 L 
    = -\sum_{n\in A_3} (\delta_n + \delta\,'_n)<0,
\end{equation*}
and know that either $\Delta_1 L<0$ or $\Delta_2 L <0$.
 Thus, our claim is justified. We conclude that any optimal solution $(\mathbf{W}^*,b^*)$ that minimizes $L$ satisfies 
 \begin{equation*}
 	b_1^*\geq b_2^* \geq \ldots \geq b_{K-1}^*.
 \end{equation*}
\end{proof}

Note that the theorem for rank-monotonicity proposed by~\cite{li2007ordinal}, in contrast to Theorem~\ref{th:ordered_thres}, requires a cost matrix $\mathcal{C}$ with each row $y_n$ being convex. Under this convexity condition, let $\lambda_{y_n}^{(k)}=|\mathcal{C}_{y_n,r_k}-\mathcal{C}_{y_n,r_{k+1}}|$ be the weight of the loss associated with the $k$-th task on the $n$-th training example, which depends on the label $y_n$. \cite{li2007ordinal} proved that by using  training example-specific task weights $\lambda_{y_n}^{(k)}$,  the optimal thresholds are ordered -- ~\cite{niu2016ordinal} noted that example-specific task weights are infeasible in practice. Moreover, this assumption requires that $\lambda_{y_n}^{(k)}\geq \lambda_{y_n}^{(k+1)}$ when $r_{k+1}< y_n$ and $\lambda_{y_n}^{(k)}\leq \lambda_{y_n}^{(k+1)}$ when $r_{k+1}>y_n$. Theorem~\ref{th:ordered_thres} is free from this requirement and allows us to choose a fixed weight for each task that does not depend on the individual training examples, which greatly reduces the training complexity. Also, Theorem~\ref{th:ordered_thres} allows for choosing either a simple uniform task weighting or taking dataset imbalances into account under the guarantee of non-decreasing predicted probabilities and consistent task predictions. Under Theorem~\ref{th:ordered_thres}, the only requirement for guaranteeing rank monotonicity is that the task weights are non-negative.

\section{Experiments}
\label{sec:experiments}

\subsection{Datasets and preprocessing}
\label{sec:datasets}

The MORPH-2 dataset~\citep{ricanek2006morph}, containing 55,608 face images, was downloaded from \url{https://www.faceaginggroup.com/morph/} and preprocessed by locating the average eye-position in the respective dataset using facial landmark detection~\citep{sagonas2016300} and then aligning each image in the dataset to the average eye position using \textit{EyepadAlign} function in MLxtend~v0.14~\citep{raschka2018mlxtend}. The faces were then re-aligned such that the tip of the nose was located in the center of each image. The age labels used in this study were in the range of 16-70 years. 

The CACD dataset~\citep{chen2014cross} was downloaded from \url{http://bcsiriuschen.github.io/CARC/} and preprocessed similar to MORPH-2 such that the faces spanned the whole image with the nose tip at the center. The total number of images is 159,449 in the age range of 14-62 years. 

The Asian Face Database (AFAD) by~\cite{niu2016ordinal} was obtained from \url{https://github.com/afad-dataset/tarball}. The AFAD database used in this study contained 165,501 faces in the range of 15-40 years. Since the faces were already centered,  no further preprocessing was required.

Following the procedure described in~\cite{niu2016ordinal}, each image database was randomly divided into 80\% training data and 20\% test data. All images were resized to 128$\times$128$\times$3 pixels and then randomly cropped to 120$\times$120$\times$3 pixels to augment the model training. During model evaluation, the 128$\times$128$\times$3 RGB face images were center-cropped to a model input size of 120$\times$120$\times$3. 

We share the training and test partitions for all datasets, along with all preprocessing code used in this paper in the code repository (Section~\ref{sec:hardware-and-software}).

\subsection{Neural network architectures}
\label{sec:architecture}

To evaluate the performance of CORAL for age estimation from face images, we chose the ResNet-34 architecture~\citep{he2016deep}, which is a modern CNN architecture that achieves good performance on a variety of image classification tasks~\citep{goceri2019analysis}. For the remainder of this paper, we refer to the original ResNet-34 CNN with standard cross-entropy loss as {CE-CNN}. To implement a ResNet-34 CNN for ordinal regression using the proposed CORAL method, we replaced the last output layer with the corresponding binary tasks (Figure~\ref{fig:resnet}) and refer to this implementation as {CORAL-CNN}. Similar to {CORAL-CNN}, we modified the output layer of ResNet-34 to implement the ordinal regression reference approach described in \citep{niu2016ordinal}; we refer to this architecture as {OR-CNN}.


\subsection{Training and evaluation}

For model evaluation and comparison, we computed the mean absolute error (MAE) and root mean squared error (RMSE), on the test set after the last training epoch:

\begin{align*}
\text{MAE} &= \frac{1}{N}\sum_{i=1}^{N} \big|y_i - h(\mathbf{x}_i)\big|,  
\\ \text{RMSE} &= \sqrt{\frac{1}{N}\sum_{i=1}^{N} \big(y_i - h(\mathbf{x}_i)\big)^2},
\end{align*}
where $y_i$ is the ground truth rank of the $i$-th test example and $h(\mathbf{x}_i)$ is the predicted rank, respectively. 
 
The model training was repeated three times with different random seeds (0, 1, and 2) for model weight initialization, while the random seeds were consistent between the different methods to allow fair comparisons. Since this study focuses on investigating rank consistency, an extensive comparison between optimization algorithms is beyond the scope of this article, so that all CNNs were trained for 200 epochs with stochastic gradient descent via adaptive moment estimation~\citep{kingma2015adam}  using exponential decay rates $\beta_0=0.90$ and $\beta_2=0.99$ (default settings) and a batch size of 256. To avoid introducing empirical bias by designing our own CNN architecture for comparing the ordinal regression approaches, we adopted a standard architecture (ResNet-34~\citep{he2016deep};  Section~\ref{sec:architecture}) for this comparison. Moreover, we chose a uniform task weighting for the cross-entropy of $K-1$ binary classifiers in CORAL-CNN, for instance, we set $\forall k: \lambda^{(k)} = 1$ in Eq.~\ref{eq:loss_fun}.

The learning rate was determined by hyperparameter tuning on the validation set. For the various losses (cross-entropy, ordinal regression CNN~\citep{niu2016ordinal}, and the proposed CORAL method), we found that a learning rate of ${\alpha=5 \times 10^{-5}}$ performed best across all models, which is likely due to using the same base architecture (ResNet-34). All models were trained for 200 epochs. From those 200 epochs, the best model was selected via MAE performance on the validation set. The selected model was then evaluated on the independent test set, from which the reported MAE and RMSE performance values were obtained. For all reported model performances, we reported the best test set performance within the 200 training epochs. We provide the complete training logs in the source code repository (Section~\ref{sec:hardware-and-software}).

\subsection{Hardware and software}
\label{sec:hardware-and-software}
All loss functions and neural network models were implemented in PyTorch 1.5~\citep{paszke2019pytorch} and trained on NVIDIA GeForce RTX 2080Ti and Titan V graphics cards. The source code is available at \url{https://github.com/Raschka-research-group/coral-cnn}.


\begin{table*}
\begin{center}
	\caption{Age prediction errors on the test sets. All models are based on the ResNet-34 architecture.}
\begin{tabular}{|l|c|c|c|c|c|c|c|} 
\hline
\multirow{2}{*}{Method} & \multicolumn{1}{c|}{\multirow{2}{*}{\begin{tabular}[c]{@{}c@{}}Random\\Seed \end{tabular}}} & \multicolumn{2}{c|}{MORPH-2} & \multicolumn{2}{c|}{AFAD} & 
\multicolumn{2}{c|}{CACD} \\ 
\cline{3-8}
 & \multicolumn{1}{c|}{} & \multicolumn{1}{c|}{MAE} & \multicolumn{1}{c|}{RMSE} & MAE & \multicolumn{1}{c|}{RMSE} 
 & MAE & \multicolumn{1}{c|}{RMSE}\\
\hline
\multirow{4}{*}{\begin{tabular}[c]{@{}l@{}}CE-CNN\end{tabular}} 
 & 0 & 3.26 & 4.62 & 3.58 & 5.01 & 5.74 & 8.20 \\
 & 1 & 3.36 & 4.77 & 3.58 & 5.01 & 5.68 & 8.09 \\
 & 2 & 3.39 & 4.84 & 3.62 & 5.06 & 5.53 & 7.92  \\ 
\cline{2-8}
 & \multicolumn{1}{l|}{{ AVG $\pm$ SD}} & \multicolumn{1}{l|}{3.34 $\pm$ 0.07} & \multicolumn{1}{l|}{4.74 $\pm$ 0.11 } & \multicolumn{1}{l|}{3.60 $\pm$ 0.02} & \multicolumn{1}{l|}{5.03 $\pm$ 0.03} & 
 \multicolumn{1}{l|}{5.65 $\pm$ 0.11} & 
 \multicolumn{1}{l|}{8.07 $\pm$ 0.14}\\
\hline
\multirow{4}{*}{\begin{tabular}[c]{@{}l@{}}OR-CNN\\ {\small\citep{niu2016ordinal}} \end{tabular}} 
 & 0 & 2.87 & 4.08 & 3.56 & 4.80 & 5.36 & 7.61\\
 & 1 & 2.81 & 3.97 & 3.48 & 4.68 & 5.40 & 7.78\\
 & 2 & 2.82 & 3.87 & 3.50 & 4.78 & 5.37 & 7.70 \\ 
\cline{2-8}
 & \multicolumn{1}{c|}{{ AVG $\pm$ SD}} & \multicolumn{1}{c|}{2.83  $\pm$ 0.03} & \multicolumn{1}{c|}{3.97 $\pm$ 0.11} & \multicolumn{1}{c|}{3.51 $\pm$ 0.04} & \multicolumn{1}{c|}{4.75 $\pm$ 0.06} &
 \multicolumn{1}{l|}{5.38 $\pm$ 0.02}& \multicolumn{1}{l|}{7.70 $\pm$ 0.09}\\ 
\hline
\multirow{4}{*}{\begin{tabular}[c]{@{}l@{}}CORAL-CNN\\ (ours) \end{tabular}} 
 & 0 & 2.66 & 3.69 & 3.42 & 4.65 & 5.25 & 7.41 \\
 & 1 & 2.64 & 3.64 & 3.51 & 4.76 & 5.25 & 7.50 \\
 & 2 & 2.62 & 3.62 & 3.48 & 4.73 & 5.24 & 7.52\\ 
\cline{2-8}
 &\multicolumn{1}{l|}{{ AVG $\pm$  SD}} & \multicolumn{1}{l|}{\textbf{2.64 $\pm$ 0.02} } & \multicolumn{1}{l|}{\textbf{3.65 $\pm$ 0.04} } & \multicolumn{1}{l|}{\textbf{3.47 $\pm$ 0.05}} & \multicolumn{1}{l|}{\textbf{4.71 $\pm$ 0.06} } &
  \multicolumn{1}{l|}{\textbf{5.25 $\pm$ 0.01} } & 
 \multicolumn{1}{l|}{\textbf{7.48 $\pm$ 0.06} }\\
\hline
\end{tabular}
\label{tab:all-results}
\end{center}
\end{table*}

\section{Results and discussion}
\label{sec:results}

We conducted a series of experiments on three independent face image datasets for age estimation (Section~\ref{sec:datasets}) to compare the proposed CORAL method (CORAL-CNN) with the ordinal regression approach proposed by~\cite{niu2016ordinal} (OR-CNN). All implementations were based on the {ResNet-34} architecture, as described in Section \ref{sec:architecture}. We include the standard ResNet-34 classification network with cross-entropy loss (CE-CNN) as a performance baseline.

\begin{figure*}
\begin{center}
\centerline{\includegraphics[width=0.75\linewidth]{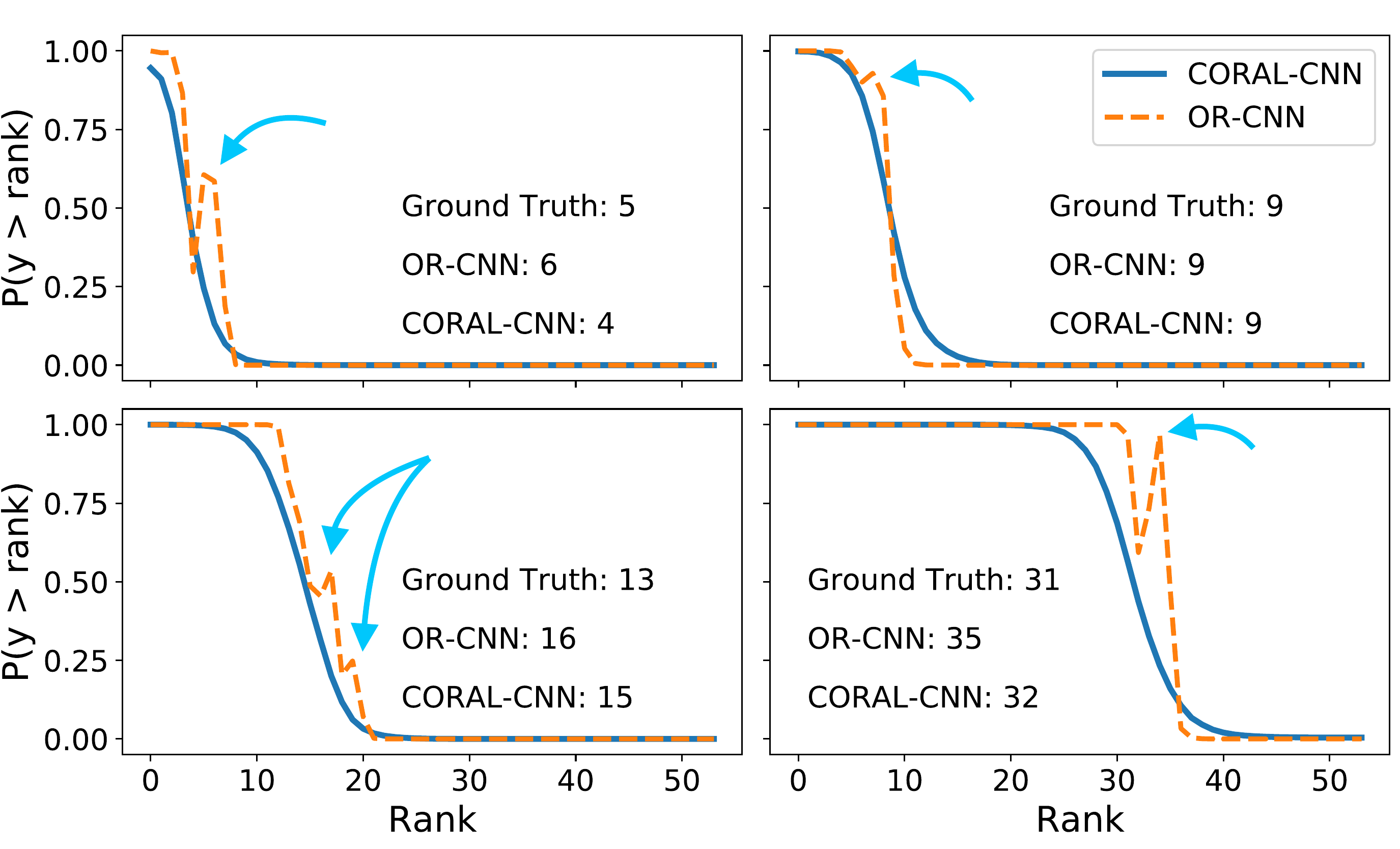}}
\caption{Graphs of the predicted probabilities for each binary classifier task on four different examples from the MORPH-2 test dataset. In all cases, OR-CNN suffers from one or more inconsistencies (indicated by arrows) in contrast to CORAL-CNN.}
\label{fig:inconsistency}
\end{center}
\end{figure*}

\begin{table*}
\centering
\caption{Average numbers of inconsistencies occurred on the different test datasets for CORAL-CNN and Niu et al's Ordinal CNN. The penultimate column and last column list the average numbers of inconsistencies focusing only on the correct and incorrect age predictions, respectively.}
\label{tab:inconsistency}
\begin{tabular}{|l|c|c|c|c|} 
\cline{2-5}
\multicolumn{1}{l|}{} & CORAL-CNN & OR-CNN \citep{niu2016ordinal} & OR-CNN \citep{niu2016ordinal}& OR-CNN \citep{niu2016ordinal}\\ 
\multicolumn{1}{l|}{} & All predictions & All predictions & Only correct predictions & Only incorrect predictions \\ 
\hline
\textbf{Morph}                 &             &               &                       &                      \\
Seed 0                & 0           & 2.28          & 1.80                  & 2.37                 \\
Seed 1                & 0           & 2.08          & 1.70                  & 2.15                 \\
Seed 2                & 0           & 0.86          & 0.65                  & 0.89                 \\ 
\hline
\textbf{AFAD}                  &             &               &                       &                      \\
Seed 0                & 0           & 1.97          & 1.88                  & 1.98                \\
Seed 1                & 0           & 1.91          & 1.81                  & 1.92                 \\
Seed 2                & 0           & 1.17        & 1.02                 & 1.19                \\ 
\hline
\textbf{CACD}                  &             &               &                       &                      \\
Seed 0                & 0           & 1.24          & 0.98                  & 1.26                 \\
Seed 1                & 0           & 1.68         & 1.29                 & 1.71                 \\
Seed 2                & 0           & 0.80         & 0.63                  & 0.81                 \\
\hline
\end{tabular}
\end{table*}

\subsection{Estimating the apparent age from face images}

Across all ordinal regression datasets (Table~\ref{tab:all-results})
we found that both OR-CNN and CORAL-CNN outperform the standard cross-entropy classification loss (CE-CNN), which does not utilize the rank ordering information. Similarly, as summarized in Table~\ref{tab:all-results}, the proposed rank-consistent CORAL method shows a substantial performance improvement over OR-CNN~\citep{niu2016ordinal}, which does not guarantee classifier consistency.

Moreover, we repeated each experiment three times using different random seeds for model weight initialization and dataset shuffling to ensure that the observed performance improvement of CORAL-CNN over OR-CNN is reproducible and not coincidental. We can conclude that guaranteed classifier consistency via CORAL has a noticeable positive effect on the predictive performance of an ordinal regression CNN (a more detailed analysis of the OR-CNN's rank inconsistency is provided in Section~\ref{sec:inconsistencies}).

For all methods (CE-CNN, CORAL-CNN, and OR-CNN), the overall performance on the different datasets appeared in the following order: MORPH-2 $>$ AFAD $>$ CACD (Table~\ref{tab:all-results}). A possible explanation is that MORPH-2 has the best overall image quality, and the photos were taken under relatively consistent lighting conditions and viewing angles. For instance, we found that AFAD includes images with very low resolutions (for example, 20x20). CACD also contains some lower-quality images. Because CACD has approximately the same size as AFAD, the overall lower performance achieved on this dataset may also be explained by the wider age range that needs to be considered (CACD: 14-62 years, AFAD: 15-40 years).

\subsection{Empirical rank inconsistency analysis}
\label{sec:inconsistencies}

By design, our proposed CORAL guarantees rank consistency (Theorem~\ref{th:ordered_thres}). In addition, we analyzed the rank inconsistency empirically for both CORAL-CNN and OR-CNN (an example of rank inconsistency is shown in Figure~\ref{fig:inconsistency}). Table~\ref{tab:inconsistency} summarizes the average numbers of rank inconsistencies for the OR-CNN and CORAL-CNN models on each test dataset. As expected, CORAL-CNN has 0 rank inconsistencies. When comparing the average numbers of rank inconsistencies considering only those cases where OR-CNN predicted the age correctly versus incorrectly, the average number of inconsistencies is higher when OR-CNN makes wrong predictions. This observation can be seen as evidence that rank inconsistency harms predictive performance. Consequently, this finding suggests that addressing rank inconsistency via CORAL is beneficial for the predictive performance of ordinal regression CNNs.

\section{Conclusions}
\label{sec:conclusions}

In this paper, we developed the CORAL framework for ordinal regression via extended binary classification with theoretical guarantees for classifier consistency. Moreover, we proved classifier consistency without requiring rank- or training label-dependent weighting schemes, which permits straightforward implementations and efficient model training. 
CORAL can be readily implemented to extend common CNN architectures for ordinal regression tasks. The experimental results showed that the CORAL framework substantially improved the predictive performance of CNNs for age estimation on three independent age estimation datasets. Our method can be readily generalized to other ordinal regression problems and different types of neural network architectures, including multilayer perceptrons and recurrent neural networks.

\section{Acknowledgments}

This research was supported by the Office of the Vice Chancellor for Research and Graduate Education at the University of Wisconsin-Madison with funding from the Wisconsin Alumni Research Foundation. Also, we thank the NVIDIA Corporation for a GPU grant to support this study.

\FloatBarrier
\clearpage
\bibliographystyle{model2-names}
\bibliography{refs}

\FloatBarrier
\clearpage

\section{Supplementary Material}

\subsection{Generalization Bounds}

Based on well-known generalization bounds for binary classification, we can derive new generalization bounds for our ordinal regression approach that apply to a wide range of practical scenarios as we only require $C_{y,r_k} = 0 \text{ if } r_k=y$ and  $C_{y,r_k} > 0 \text{ if } r_k \neq y$. Moreover, Theorem~\ref{th:gener-error} shows that if each binary classification task in our model generalizes well in terms of the standard 0/1-loss, the final rank prediction via $h$ (Eq.~\ref{eq:predicted-label}) also generalizes well.
\begin{theorem}[reduction of generalization error]\label{th:gener-error}
Suppose $\mathcal{C}$ is the cost matrix of the original ordinal label prediction problem, with $\mathcal{C}_{y,y}=0$ and $\mathcal{C}_{y,r_k}>0$ for $k\neq y$. $P$ is the underlying distribution of $(\mathbf{x},y)$, for instance, $(\mathbf{x},y)\sim P$.  If the binary classification rules $\{f_k\}_{k=1}^{K-1}$ obtained by optimizing Eq.~\ref{eq:loss_fun} are rank-monotonic, then

\begin{equation}\label{eq:gen-bound1}
\resizebox{.99\hsize}{!}{$\underset{(\mathbf{x},y)\sim P}{\mathbb{E}}\mathcal{C}_{y,h(\mathbf{x})}
      \leq 
     \sum_{k=1}^{K-1}\big|\mathcal{C}_{y,r_k}-\mathcal{C}_{y,r_{k+1}}\big| \underset{(\mathbf{x},y)\sim P}{\mathbb{E}}\mathbbm{1}\{f_k(\mathbf{x})\neq y^{(k)}\}$}.
\end{equation}

\end{theorem}
\begin{proof}
For any $\mathbf{x}\in \mathcal{X}$, we have
\begin{equation*}
    f_1(\mathbf{x})\geq f_2(\mathbf{x}) \geq \ldots \geq f_{K-1}(\mathbf{x}).
\end{equation*}

\noindent If $h(\mathbf{x})=y$, then $\mathcal{C}_{y,h(\mathbf{x})}=0$.\\
If $h(\mathbf{x})=r_q\prec y=r_s$, then $q<s$. We have 
\begin{equation*}
f_1(\mathbf{x})=f_2(\mathbf{x})=\ldots=f_{q-1}(x)=1	
\end{equation*}
\noindent and 
\begin{equation*}
f_q(\mathbf{x})=f_{q+1}(\mathbf{x})=\ldots=f_{K-1}(\mathbf{x})=0.
\end{equation*}
\noindent Also, 
\begin{equation*}
y^{(1)}=y^{(2)}=\ldots=y^{(s-1)}=1
\end{equation*}
and 
\begin{equation*}
y^{(s)}=y^{(s+1)}=\ldots=y^{(K-1)}=0.
\end{equation*}
Thus, $\mathbbm{1}\{f_k(\mathbf{x})\neq y^{(k)}\}=1$ if and only if $q\leq k\leq s-1$. Since $\mathcal{C}_{y,y}=0,$

\begin{align*}
    \mathcal{C}_{y,h(\mathbf{x})} & =  \sum_{k=q}^{s-1}(\mathcal{C}_{y,r_k}-\mathcal{C}_{y,r_{k+1}})\cdot \mathbbm{1}\{f_k(\mathbf{x})\neq y^{(k)}\} \\
    & \leq  \sum_{k=q}^{s-1}\big|\mathcal{C}_{y,r_k}-\mathcal{C}_{y,r_{k+1}}\big|\cdot \mathbbm{1}\{f_k(\mathbf{x})\neq y^{(k)}\} \\
    & \leq  \sum_{k=1}^{K-1}\big|\mathcal{C}_{y,r_k}-\mathcal{C}_{y,r_{k+1}}\big|\cdot \mathbbm{1}\{f_k(\mathbf{x})\neq y^{(k)}\}.
\end{align*}

Similarly, if $h(x)=r_q\succ y=r_s$, then $q>s$ and
\begin{align*}
    \mathcal{C}_{y,h(\mathbf{x})}  &=  \sum_{k=s}^{q-1}(\mathcal{C}_{y,r_{k+1}}-\mathcal{C}_{y,r_{k}})\cdot \mathbbm{1}\{f_k(\mathbf{x})\neq y^{(k)}\} \\
    & \leq \sum_{k=1}^{K-1}\big|\mathcal{C}_{y,r_{k+1}}-\mathcal{C}_{y,r_{k}}\big|\cdot \mathbbm{1}\{f_k(\mathbf{x})\neq y^{(k)}\}.
\end{align*}
In any case, we have 
\begin{equation*}
    \mathcal{C}_{y,h(\mathbf{x})}\leq \sum_{k=1}^{K-1}\big|\mathcal{C}_{y,r_k}-\mathcal{C}_{y,r_{k+1}}\big|\cdot \mathbbm{1}\{f_k(\mathbf{x})=y^{(k)}\}.
\end{equation*}
By taking the expectation on both sides with $(\mathbf{x},y)\sim P$, we arrive at Eq.~\eqref{eq:gen-bound1}.
\end{proof}

 In \cite{li2007ordinal}, by assuming the cost matrix to have V-shaped rows, the researchers define generalization bounds by constructing a discrete distribution on $\{1,2,\ldots,K-1\}$ conditional on each $y$, given that the binary classifications are rank-monotonic or every row of $\mathcal{C}$ is convex. However, the only case they provided for the existence of rank-monotonic binary classifiers was the ordered threshold model, which requires a cost matrix with convex rows and example-specific task weights. In other words, when the cost matrix is only V-shaped but does not meet the convex row condition, for instance, $\mathcal{C}_{y,r_k}-\mathcal{C}_{y,r_{k-1}}>\mathcal{C}_{y,r_{k+1}}-\mathcal{C}_{y,r_k}>0$ for some $r_k>y$, the method proposed in \cite{li2007ordinal} did not provide a practical way to bound the generalization error. 
Consequently, our result does not rely on cost matrices with V-shaped or convex rows and can be applied to a broader variety of real-world use cases.

\end{document}